\documentclass[a4paper, 11pt]{amsart}
\usepackage{fullpage}

\usepackage{wrapfig}

\usepackage[T1]{fontenc}    
\usepackage{url}            
\usepackage{booktabs}       
\usepackage{amsfonts}       
\usepackage{nicefrac}       
\usepackage{microtype}      

\usepackage{algorithm}
\usepackage{algorithmic}

\makeatletter
\newcommand{\para}[1]{%
  \par
  \addvspace{\medskipamount}
  \textit{#1\@addpunct{.}}\enspace\ignorespaces
}
\makeatother

\usepackage{amsmath}
\usepackage{amsthm}
\usepackage{mathtools} 
\usepackage{algorithm}
\usepackage{enumitem}
\usepackage{bbm}
\usepackage[bookmarks=true,hypertexnames=false,pagebackref]{hyperref}
\hypersetup{colorlinks=true, citecolor=blue, linkcolor=red, urlcolor=blue}

\newcommand{\abs}[1]{\left\vert#1\right\vert}
\newcommand{\norm}[2]{\left\|#2\right\|_{#1}}
\newcommand{\inner}[2]{\langle#2\rangle_{#1}}

\newtheorem{theorem}{Theorem}

\newtheorem{lemma}[theorem]{Lemma}

\newtheorem{corollary}[theorem]{Corollary}
\newtheorem*{remark}{Remark}
\newtheorem{example}[theorem]{Example}

\newtheorem{definition}[theorem]{Definition}
\newtheorem{proposition}[theorem]{Proposition}

\usepackage[round]{natbib}

\title{Layerwise Systematic Scan: \\
Deep Boltzmann Machines and Beyond}

\author{Heng Guo}
\address[Heng Guo]{School of Informatics, University of Edinburgh, United Kingdom.}
\email{h.guo@inf.ed.ac.uk}

\author{Kaan Kara}
\address[Kaan Kara]{Department of Computer Science, ETH Zurich, Switzerland}
\email{kaan.kara@inf.ethz.ch}

\author{Ce Zhang}
\address[Ce Zhang]{Department of Computer Science, ETH Zurich, Switzerland}
\email{ce.zhang@inf.ethz.ch}

\begin{document}

\begin{abstract}
For Markov chain Monte Carlo methods, one of the greatest discrepancies between theory and system is the \emph{scan order} ---
while most theoretical development on the mixing time analysis deals with random updates, real-world systems are implemented with systematic scans.
We bridge this gap for models that exhibit a bipartite structure, including, most notably, the Restricted/Deep Boltzmann Machine.
The de facto implementation for these models scans variables in a layer-wise fashion.
We show that the Gibbs sampler with a layerwise alternating scan order has its relaxation time (in terms of epochs) no larger than that of a random-update Gibbs sampler (in terms of variable updates).
We also construct examples to show that this bound is asymptotically tight.
Through standard inequalities, our result also implies a comparison on the mixing times.
\end{abstract}

\maketitle

\section{Introduction}

Gibbs sampling, or the Markov chain Monte Carlo method in general,
plays a central role in machine learning and 
have been widely implemented as the backbone algorithm for models such as Deep Boltzmann Machines~\citep{Salakhutdinov:2009:AISTATS}, 
latent Dirichlet allocations~\citep{Blei:2003:JMLR}, and factor graphs in general.
Given a set of random variables and a target distribution $\pi$, 
the Gibbs sampler iteratively updates one variable at a time according to the distribution $\pi$ conditioned on the values of all other variables.
If the ergodicity condition is met, then the Gibbs sampler eventually converges to the target distribution.

There are two ways to choose which variable to update
at the next iteration: 
(1) {\em Random Update}, where in each epoch (or round) one variable is picked uniformly at random with replacement; 
and (2) {\em Systematic Scan}, where in each epoch all variables are updated using some pre-determined order.
Although most theoretical development on analyzing Gibbs sampling deals with random updates \citep{Jerrum03,LPW06},
systematic scans are prevalent in real-world implementations due to their hardware-friendly nature (cache locality for factor graphs, SIMD for Deep Boltzmann Machines, etc.). 
It is natural to wonder, {\em whether using systematic scan, rather than random updates, would delay the mixing time, the number of iterations the Gibbs sampler requires to reach the target distribution.}

The mixing time of these two update strategies can differ by some high polynomial factors in either directions \citep{HDMR16,RR15}.
Even more pathological examples were constructed for non-Gibbs Markov chains such that systematic scan is not even ergodic whereas the random-update sampler is rapidly mixing~\citep{DGJ08}.
Indeed, even for a system as simple as the Ising model, a comparison result remains elusive \citep[Open problem 5, p.~300]{LPW06}.
As a consequence, theoretical results on rapidly mixing, such as \citep{BD97,MS13}, do not readily apply to the scan algorithms used in practice.

\subsection{Main results.}
In this paper, we bridge this gap between theory and system.
We focus on \emph{bipartite distributions}, 
in which variables can be divided into two partitions ---
conditioned on one of the partitions, variables from the other partition are mutually independent.
This bipartite structure arises naturally in practice, including Restricted/Deep Boltzmann Machines.
For a bipartite distribution, the de facto implementation is that in each epoch, 
we scan all variables from one of the partitions first, and then the other.
We call this the \emph{alternating-scan} sampler.
Note that in order to define a valid Markov chain, we have to consider systematic scans in epochs, in which all variables are updated once.
Our main theorem is the following.

\begin{theorem}[{\bf Main Theorem}] \label{thm:AS}
  For any bipartite distribution $\pi$,
  if the random-update Gibbs sampler is ergodic, 
  then so is the alternating-scan sampler. 
  Moreover, the relaxation time of the alternating-scan sampler
  (in terms of epochs)
  is no larger than that of the random-update one
  (in terms of variable updates).
\end{theorem}

The relaxation time (inverse spectral gap) measures the mixing time from a ``warm'' start.
It is closely related to the (total variation distance) mixing time,
and governs mixing times under other metrics as well~\citep{LPW06}.
Through standard inequalities, Theorem \ref{thm:AS} also implies a comparison result in terms of mixing times, Corollary \ref{cor:AS}.
As we count epochs in Theorem \ref{thm:AS}, the alternating-scan sampler is implicitly slower by a factor of $n$, the number of variables.
We also show that Theorem \ref{thm:AS} is asymptotically tight via Example \ref{exm:bipartite-complete}.
Thus this implicit factor $n$ slowdown cannot be improved in general.

More specifically, we summarize our contribution as follows.

\begin{enumerate}
  \item In Section~\ref{sec:AS}, we establish Theorem \ref{thm:AS}. 
    By focusing on bipartite systems, we are able to obtain much stronger result than recent studies in the more general setting~\citep{HDMR16}.
    We note that standard Markov chain comparison results, such as \citep{DS93}, do not seem to fit into our setting.
    Instead, we give a novel analysis via estimates of operator norms of certain carefully defined matrices.
    One key observation is to consider an artificial but equivalent variant of the alternating-scan sampler,
    where we insert an extra random update between updating variables from the two partitions.
    This does not change the algorithm since the extra random update is either redundant with the updates in the first partition or with those in the second.
  \item In Section~\ref{sec:implication}, 
    we discuss bipartite distributions that arise naturally in machine learning. 
    In particular, our result is a rigorous justification of the popular layer-wise scan sampler for Deep Boltzmann Machines~\citep{Salakhutdinov:2009:AISTATS}.
    Our result also applies to other models such as Restricted Boltzmann Machines~\citep{Smolensky:1986:IPD} and, more generally, any bipartite factor graph.
  \item In Section~\ref{sec:experiments}, 
    we conduct experiments to verify our theory and analyze the gap between our worst case theoretical bound and numerical evidences. 
    We observe that in the rapidly mixing regime, the alternating-scan sampler is usually faster than the random-update one,
    whereas in the slow mixing regime, the alternating-scan sampler can be slower by a factor $O(n)$.
    We hope these observations shed some light on more fine-grained comparison bounds in the future.
\end{enumerate}

\section{Related Work}

Probably the most relevant work is the recent
analysis conducted by \cite{HDMR16} about
the impact of the scan order on the mixing time of 
the Gibbs sampling. They (1) constructed a variety of 
models in which the scan order can change the mixing 
time significantly in several different ways and (2)
proved comparison results on the mixing time between random updates and a variant of systematic scans where ``lazy'' 
moves are allowed. In this paper, we focus on
a more specific case, i.e., bipartite systems,
and so our bound is stronger --- in fact, our bound
can be exponentially stronger when the underlying chain is torpidly mixing.
Moreover, our result does not modify the standard scan algorithm.

Another related work is the recent analysis by
\cite{Tosh16} considering the mixing time 
of an alternating sampler for the Restricted Boltzmann Machine (RBM). 
Tosh showed that, under Dobrushin-like conditions~\citep{Dobrushin70},
i.e., when the weights in the RBM are sufficiently small, the alternating sampler mixes rapidly.
For models other than RBM, mixing time results for systematic scans are relatively rare.
Known examples are usually restricted to very specific models~\citep{DR00} 
or under conditions to ensure that the correlations 
are sufficiently weak~\citep{DGJ06,Hayes06,DGJ08}.
Typical conditions of this sort are variants of the classical Dobrushin 
condition~\citep{Dobrushin70}. 
See also \citep{BCSV18} for very recent results on analyzing the alternating scan sampler (among others) on the 2D grid 
under conditions of the Dobrushin-type.
In contrast, our work focuses on the relative
performance between random updates and systematic scan,
and does not rely on Dobrushin-like conditions.
In particular, our results extend to the torpid mixing regime as well as the rapid mixing one.

Our primary focus is on discrete state spaces.
The scan order question has also been asked and explored in general state spaces.
Despite a long line of research \citep{Has70,Pes73,CPS90,LWK95,RS97,RR97,Tie98,MRJ14,RR15,And16}, 
to the best of our knowledge, no decisive answer is known.

Another line of related research is about the scan order in \emph{stochastic gradient descent}~\citep{RR12,Sha16,GOP17}.
Our setting in this paper is very different and the techniques are different as well.

\section{Preliminaries on Markov Chains}

Let $\Omega$ be a discrete state space and $P$ be a $\abs{\Omega}$-by-$\abs{\Omega}$ stochastic matrix describing a (discrete time) Markov chain on $\Omega$.
The matrix $P$ is also called the transition matrix or the kernel of the chain.
Thus, $P^t(\sigma_0,\cdot)$ is the distribution of the chain at time $t$ starting from $\sigma_0$.
Let $\pi(\cdot)$ be a stationary distribution of $P$.
The Markov chain defined by $P$ is \emph{reversible} (with respect to $\pi(\cdot)$) if $P$ satisfies the detailed balance condition:
\begin{align}\label{eqn:detailed-balance}
  \pi(\sigma)P(\sigma,\tau) = \pi(\tau)P(\tau,\sigma)
\end{align}
for any $\sigma,\tau\in \Omega$.
We note that in general the systematic scan sampler is not reversible.
The Markov chain is called \emph{irreducible} if $P$ connects the whole state space $\Omega$,
namely, for any $\sigma,\tau\in\Omega$, there exists $t$ such that $P^t(\sigma,\tau)>0$.
It is called \emph{aperiodic} if $\gcd\{t>0:P^t(\sigma,\sigma)>0\}=1$ for every $\sigma\in\Omega$.
We call $P$ \emph{ergodic} if it is both irreducible and aperiodic.
An ergodic Markov chain converges to its unique stationary distribution \citep{LPW06}.

The \emph{total variation} distance $\norm{TV}{\cdot}$ for two distributions $\mu$ and $\nu$ on $\Omega$ is defined as
\begin{align*}
  \norm{TV}{\mu-\nu} = \max_{A\subset\Omega}\abs{\mu(A)-\nu(A)}=\frac{1}{2}\sum_{\sigma\in\Omega}\abs{\mu(\sigma)-\nu(\sigma)}.
\end{align*}
The mixing time $T_{mix}$ is defined as 
\begin{align*}
  T_{mix}(P):=\min\left\{t\ge 0:\max_{\sigma\in \Omega}\norm{TV}{P^t(\sigma,\cdot)-\pi}\le \frac{1}{2e}\right\},
\end{align*}
where the choice of the constant $\frac{1}{2e}$ is merely for convenience and is not significant \citep{LPW06}.

When $P$ is ergodic and reversible, the eigenvalues $(\xi_i)_{i\in[\abs{\Omega}]}$ of $P$ satisfies $-1<\xi_i\le 1$,
and additionally, $Pf=f$ if and only if $f$ is constant (see \citep[Lemma 12.1]{LPW06}).
The \emph{spectral gap} of $P$ is defined by
\begin{align}\label{eqn:spectral-gap}
  \lambda(P):=1-\max\{\abs{\xi}:\xi\text{ is an eigenvalue of $P$}\notag\\
  \hspace{1cm}\text{ and $\xi\neq 1$}\}.
\end{align}
The \emph{relaxation time} for a reversible $P$
is defined as 
\begin{align}  \label{eqn:relaxation-time}
  T_{rel}(P):=\lambda(P)^{-1}.
\end{align}
The relaxation time and the mixing time differ by at most a factor of $\log \left(  \frac{2e}{\pi_{min}}\right)$ where $\pi_{min}=\min_{\sigma\in \Omega}\pi(\sigma)$, 
shown by the following theorem (see, for example, \citep[Theorem 12.4 and 12.5]{LPW06}).
In fact, the relaxation time governs mixing properties with respect to metrics other than the total variation distance as well.
See \citep[Chapter 12]{LPW06} for more details.

\begin{theorem}\label{thm:relaxation}
  Let $P$ be the transition matrix of a reversible and ergodic Markov chain with the state space $\Omega$ and the stationary distribution $\pi$.
  Then
  \begin{align*}
    T_{rel}(P)-1\le T_{mix}(P)\le T_{rel}(P) \log \left(  \frac{2e}{\pi_{min}}\right),
  \end{align*}
  where $\pi_{min}=\min_{\sigma\in \Omega}\pi(\sigma)$.
\end{theorem}
The factor $\log\pi_{min}^{-1}$ is usually linear in $n$, the number of variables,
in the context of Gibbs sampling which is our primary focus later.
Theorem~\ref{thm:relaxation} is tight, and there is no good way of avoiding losing this $\log\pi_{min}^{-1}$ factor in general,
with the spectral method.

Unfortunately, the systematic-scan sampler is not reversible, and therefore Theorem \ref{thm:relaxation} does not apply.
Instead, we use an extension developed by \cite{Fill91}.
For a non-reversible transition matrix $P$,
let the \emph{multiplicative reversiblization} be $R(P):= P P^*$, where $P^*$ is the \emph{adjoint} of $P$ defined as
\begin{align}  \label{eqn:adjoint}
  P^*(\sigma,\tau) = \frac{\pi(\tau)P(\tau,\sigma)}{\pi(\sigma)}.  
\end{align}
Then $R(P)$ is reversible.
Let the \emph{relaxation time} for a (not necessarily reversible) $P$ be
\begin{align}  \label{eqn:relaxation-non-reverse}
  T_{rel}(P):= \frac{1}{1 - \sqrt{1- \lambda(R(P))}}.
\end{align}
In particular, if $P$ is reversible, then \eqref{eqn:relaxation-non-reverse} recovers \eqref{eqn:relaxation-time}
(see Proposition \ref{prop:non-reverse-relax}).
In general, the multiplicative reversibilization mixes similarly to the original non-reversible chain.
See \citep{Fill91} for more details.

The following theorem is a simple consequence of \cite[Theorem 2.1]{Fill91}.

\begin{theorem} \label{thm:relaxation-reverse}
  Let $P$ be the transition matrix of an ergodic Markov chain with the state space $\Omega$ and the stationary distribution $\pi$.
  Then
  \begin{align*}
    T_{mix}(P)\le \log \left(  \frac{4e^2}{\pi_{min}}\right) T_{rel}(P),
  \end{align*}
  where $\pi_{min}=\min_{\sigma\in \Omega}\pi(\sigma)$.
\end{theorem}

Note that our definition of relaxation times \eqref{eqn:relaxation-non-reverse} for non-reversible Markov chains yields asymptotically the same upper bound in Theorem \ref{thm:relaxation}.

\begin{proof}[Proof of Theorem \ref{thm:relaxation-reverse}]
  We first restate \cite[Theorem 2.1]{Fill91} (note that the norm in \citep{Fill91} is twice the total variation distance):
  \begin{align}\label{eqn:Fill}
    \norm{TV}{P^t(\sigma,\cdot)-\pi}^2\le \frac{\left(1- \lambda(R(P)) \right)^t}{\pi(\sigma)}.
  \end{align}
  Let $\lambda:= \lambda(R(P))$ and $T:=\log\left( \frac{4e^2}{\pi_{\min}} \right) T_{rel}(P)
  = \frac{1}{1 - \sqrt{1- \lambda}}\log\left( \frac{4e^2}{\pi_{\min}} \right)$.
  Then it is easy to verify that
  \begin{align*}
    T \ge \frac{2}{\lambda}\log\left( \frac{2e}{\sqrt{\pi_{\min}}} \right)
  \end{align*}
  and by \eqref{eqn:Fill}, we have that
  \begin{align*}
    \max_{\sigma\in\Omega}\norm{TV}{P^T(\sigma,\cdot)-\pi} & \le \frac{\left(1- \lambda \right)^{T/2}}{\sqrt{\pi_{\min}}}\\
    & \le \frac{\left(1- \lambda \right)^{\lambda^{-1}\log\left( \frac{2e}{\sqrt{\pi_{\min}}} \right)}}{\sqrt{\pi_{\min}}} \\
    & \le \frac{e^{-\log\left( \frac{2e}{\sqrt{\pi_{\min}}} \right)}}{\sqrt{\pi_{\min}}}\\
    & = \frac{1}{2e}.
  \end{align*}
  In other words,
  \begin{align*}
    T_{mix}(P)\le T & = \log\left( \frac{4e^2}{\pi_{\min}} \right)T_{rel}(P). \qedhere
  \end{align*}
\end{proof}

\subsection{Operator Norms and the Spectral Gap}

We also view the transition matrix $P$ as an operator that mapping functions to functions.
More precisely, let $f$ be a function $f:\Omega\rightarrow \mathbb{R}$ and $P$ acting on $f$ is defined as
\begin{align*}
  Pf(x):=\sum_{y\in\Omega}P(x,y)f(y).
\end{align*}
This is also called the \emph{Markov operator} corresponding to $P$.
We will not distinguish the matrix $P$ from the operator $P$ as it will be clear from the context.
Note that $Pf(x)$ is the expectation of $f$ with respect to the distribution $P(x,\cdot)$.
We can regard a function $f$ as a column vector in $\mathbb{R}^{\Omega}$,
in which case $Pf$ is simply matrix multiplication.
Recall \eqref{eqn:adjoint} and $P^*$ is also called the \emph{adjoint operator} of $P$.
Indeed, $P^*$ is the (unique) operator that satisfies $\inner{\pi}{f,Pg} = \inner{\pi}{P^*f,g}$.
It is easy to verify that if $P$ satisfies the detailed balanced condition \eqref{eqn:detailed-balance},
then $P$ is \emph{self-adjoint}, namely $P=P^*$.

The Hilbert space $L_2(\pi)$ is given by endowing $\mathbb{R}^{\Omega}$ with the inner product
\begin{align*}
  \inner{\pi}{f,g}:=\sum_{x\in\Omega}f(x)g(x)\pi(x), 
\end{align*}
where $f,g\in\mathbb{R}^{\Omega}$.
In particular, the norm in $L_2(\pi)$ is given by
\begin{align*}
  \norm{\pi}{f} := \inner{\pi}{f,f}.
\end{align*}

The spectral gap \eqref{eqn:spectral-gap} can be rewritten in terms of the operator norm of $P$, 
which is defined by 
\begin{align*}
  \norm{\pi}{P}:=\max_{\norm{\pi}{f}\neq 0} \frac{\norm{\pi}{Pf}}{\norm{\pi}{f}}.
\end{align*}
Indeed, the operator norm equals the largest eigenvalue (which is just $1$ for a transition matrix $P$),
but we are interested in the second largest eigenvalue.
Define the following operator
\begin{align}  \label{eqn:S_pi}
  S_{\pi}(\sigma,\tau) := \pi(\tau).
\end{align}
It is easy to verify that $S_{\pi}f(\sigma) = \inner{\pi}{f,\mathbf{1}}$ for any $\sigma$.
Thus, the only eigenvalues of $S_{\pi}$ are $0$ and $1$,
and the eigenspace of eigenvalue $0$ is $\{f\in L_2(\pi):\inner{\pi}{f,\mathbf{1}}=0\}$.
This is exactly the union of eigenspaces of $P$ excluding the eigenvalue $1$.
Hence, the operator norm of $P-S_{\pi}$ equals the second largest eigenvalue of $P$, namely,
\begin{align}\label{eqn:spectral-operator}
  \lambda(P) = 1- \norm{\pi}{P-S_{\pi}}.
\end{align}
The expression in \eqref{eqn:spectral-operator} can be found in, for example, \cite[Eq.~(2.8)]{Ullrich:dissertation}.
In particular, using \eqref{eqn:spectral-operator}, 
we show that the definition \eqref{eqn:relaxation-non-reverse} coincides with \eqref{eqn:relaxation-time} when $P$ is reversible.

\begin{proposition}\label{prop:non-reverse-relax}
  Let $P$ be the transition matrix of a reversible matrix with the stationary distribution $\pi$. 
  Then 
  \begin{align*}
    \frac{1}{\lambda(P)} = \frac{1}{1 - \sqrt{1- \lambda(R(P))}}.
  \end{align*}
\end{proposition}
\begin{proof}
  Since $P$ is reversible, $P$ is self-adjoint, namely, $P^*=P$.
  Hence $\left( P-S_{\pi} \right)^* = P^* - S_{\pi}$ and
  \begin{align*}
    \left( P-S_{\pi} \right)\left( P - S_{\pi} \right)^* & = \left( P-S_{\pi} \right)\left( P^* - S_{\pi} \right) \\
    & = PP^* - PS_{\pi} - S_{\pi}P^*+S_{\pi}S_{\pi} \\
    & = PP^* - S_{\pi},
  \end{align*}
  where we use the fact that $PS_{\pi}=S_{\pi}P^*=S_{\pi}S_{\pi}=S_{\pi}$.
  It implies that
  \begin{align*}
    1- \lambda(R(P)) & = \norm{\pi}{R(P)-S_{\pi}} \tag{by \eqref{eqn:spectral-operator}}\\
    & = \norm{\pi}{PP^*-S_{\pi}} \\
    & = \norm{\pi}{\left( P-S_{\pi} \right)\left( P - S_{\pi} \right)^*} \\
    & = \norm{\pi}{ P-S_{\pi}}^2 \\
    & = \left( 1-\lambda(P) \right)^2.
  \end{align*}
  Rearranging the terms yields the claim.
\end{proof}

\section{Alternating Scan}\label{sec:AS}

In this section we describe the random update and the alternating scan sampler,
and compare these two.
Let $V=\{x_1,\dots,x_n\}$ be a set of variables where each variable takes values from some finite set $S$.
Let $\pi(\cdot)$ be a distribution defined on $S^V$.

Let $\sigma\in S^{V}$ be a configuration, namely $\sigma:V\to S$.
Let $\sigma^{x,s}$ be the configuration that agrees with $\sigma$ except at $x$, where $\sigma^{x,s}(x)=s$ for $s\in S$.
In other words, for any $y\in V$,
\begin{align*}
  \sigma^{x,s}(y) := 
  \begin{cases}
    \sigma(y) & \textrm{if $y\neq x$};\\
    s & \textrm{if $y = x$}.
  \end{cases}
\end{align*}

\begin{algorithm}[t!]
  \caption{Gibbs sampling with random updates}
  \label{alg:Gibbs}
  \begin{algorithmic}
    \REQUIRE {Starting configuration $\sigma=\sigma_0$}
    \FOR {$t= 1,\dots, T_{mix}$}
    \STATE {With probability $1/2$, do nothing.}
    \STATE {Otherwise, select a variable $x\in V$ uniformly at random.}
    \STATE {Set $\sigma \gets \sigma^{x,s}$ with probability $\frac{\pi(\sigma^{x,s})}{\sum_{t\in S}\pi(\sigma^{x,t})}$.}
    \ENDFOR
    \RETURN {$\sigma$}
  \end{algorithmic}
\end{algorithm}

The lazy\footnote{We choose to present the lazy sampler due to its popularity in theoretical analysis. Our arguments later in fact also apply to non-lazy samplers as well. See the remarks after the proof of Theorem \ref{thm:AS}.} Gibbs sampler is defined in Algorithm \ref{alg:Gibbs}.
Let $n=\abs{V}$ be the total number of variables.
The transition kernel $P_{RU}$ (where RU stands for ``random updates'') of the sampler in Algorithm \ref{alg:Gibbs} is defined as:
\begin{align}
  P_{RU}(\sigma,\tau) =
  \begin{cases}
    \frac{1}{2n} \cdot \frac{\pi(\sigma^{x,s})}{\sum_{t\in S}\pi(\sigma^{x,t})} & \textrm{ if } \tau\neq\sigma \textrm{ and there are $x\in V$} \\
    & \textrm{ and $s\in S$ such that } \tau = \sigma^{x,s};\\
    1/2 + \sum_{x\in V}\frac{1}{2n} \cdot \frac{\pi(\sigma^{x,\sigma(x)})}{\sum_{t\in S}\pi(\sigma^{x,t})} & \textrm{ if } \tau=\sigma;\\
    0 & \textrm{ otherwise,}
  \end{cases} \notag
\end{align}
where $\sigma,\tau$ are two configurations.
It is not hard to see, for example, by checking the detailed balance condition \eqref{eqn:detailed-balance}, 
that $\pi(\cdot)$ is the stationary distribution of $P_{RU}$.
Note that this Markov chain is {\it lazy}, i.e., it remains at its current state with probability at least~$1/2$.
This self-loop probability is higher than $1/2$ because when we update a variable there is positive probability of no change.
Lazy chains are often studied in the literature because of its technical conveniences.
The self-loop eliminates potential periodicity,
and all eigenvalues of a lazy chain are non-negative.
In the context of Gibbs sampling, these are merely artifacts of the available techniques 
and considering the lazy version is not really necessary \citep{RU13,DGU14}.
Our main result actually applies to both lazy and non-lazy versions.
See the remarks after the proof of Theorem \ref{thm:AS}.

Our main focus is bipartite distributions, defined next.
These distributions arise naturally from bipartite factor graphs, including, most notably, Restricted Boltzmann Machines.

\begin{definition}\label{def:bipartite}
  The joint distribution $\pi(\cdot)$ of random variables $V=\{x_1,\dots,x_n\}$ is \emph{bipartite},
  if $V$ can be partitioned into two sets $V_1$ and $V_2$ $($namely $V_1\cup V_2=V$ and $V_1\cap V_2=\emptyset)$,
  such that conditioned on any assignment of variables in $V_2$, all variables in $V_1$ are mutually independent, and vice versa.
\end{definition}

In the following we consider a particular systematic scan sampler for bipartite distributions.
For a configuration $\sigma$, let $\sigma_i := \sigma|_{V_i}$ be its projection on $V_i$ where $i=1,2$.
The alternating-scan sampler is given in Algorithm \ref{alg:Alternating},
where $n_1=\abs{V_1}$ and $n_2=\abs{V_2}$.

\begin{algorithm}[t!]
  \caption{Alternating-scan sampler}
  \label{alg:Alternating}
  \begin{algorithmic}
    \REQUIRE {Starting configuration $\sigma=\sigma_0$}
    \FOR {$t= 1,\dots, T_{mix}$}
    \FOR {$i=1,\dots,n_1$}
      \STATE {Set $\sigma \gets \sigma^{x_i,s}$ with probability $\frac{\pi(\sigma^{x_i,s})}{\sum_{t\in S}\pi(\sigma^{x_i,t})}$.}
    \ENDFOR
    \FOR {$j=1,\dots,n_2$}
      \STATE {Set $\sigma \gets \sigma^{y_j,s}$ with probability $\frac{\pi(\sigma^{y_j,s})}{\sum_{t\in S}\pi(\sigma^{y_j,t})}$.}
    \ENDFOR
    \ENDFOR
    \RETURN {$\sigma$}
  \end{algorithmic}
\end{algorithm}

In other words, the alternating-scan sampler sequentially resamples all variables in $V_1$, and then resamples all variables in $V_2$.
Note that since we are considering a bipartite distribution,
in order to resample $x_i \in V_1$, we only need to condition on $\sigma_2$.
In other words, for any $i\in[n_1]$, 
the distribution $\left( \frac{\pi(\sigma^{x_i,s})}{\sum_{t\in S}\pi(\sigma^{x_i,t})} \right)_{s\in S}$ that we draws from depends only on $\sigma_2$.
Similarly, resampling $y_j\in V_2$ only depends on $\sigma_1$.
We will denote the transition kernel of the alternating-scan sampler as $P_{AS}$,
where AS stands for ``alternating scan''.

An unusual feature of systematic-scan samplers (including the alternating-scan sampler) is that they are not reversible.
Namely the detailed balance condition \eqref{eqn:detailed-balance} does not in general hold.
This is because updating variables $x$ and $y$ in order is in general different from updating $y$ and $x$ in order.
This imposes a technical difficulty as most of the theoretical tools of analyzing these chains are not suitable for irreversible chains,
such as the Dirichlet form \citep{DS93} or conductance bounds~\citep{JSising,Sin92}.

On the other hand, the scan sampler is aperiodic.
Any potential state $\sigma$ of the chain must be in the state space $\Omega$.
Therefore $\pi(\sigma)>0$ and the probability of staying in $\sigma$ is strictly positive.
Moreover, if the Gibbs sampler is irreducible (namely the state space $\Omega$ is connected via single variable flips), then so is the scan sampler.
This is because any single variable update can be simulated in the scan sampler, with small but strictly positive probability.
Hence if the Gibbs sampler is ergodic, then so is the scan sampler.

We restate our main theorem here in formal terms.

{\renewcommand{\thetheorem}{\ref{thm:AS}}
\begin{theorem}
  For any bipartite distribution $\pi$,
  if $P_{RU}$ is ergodic, then so is $P_{AS}$.
  Moreover, 
  \begin{align*}
    T_{rel}(P_{AS})\le T_{rel}(P_{RU}).
  \end{align*}
\end{theorem}
\addtocounter{theorem}{-1}
}

We will prove Theorem \ref{thm:AS} next.
The transition matrix of updating a particular variable $x$ is the following
\begin{align}  \label{eqn:update-v}
  T_x(\sigma,\tau)=
  \begin{cases}
    \frac{\pi(\sigma^{x,s})}{\sum_{s\in S}\pi(\sigma^{x,s})} & \textrm{if } \tau = \sigma^{x,s} \textrm{ for some $s\in S$};\\
    0 & \textrm{otherwise.}\\
  \end{cases}
\end{align}
Moreover, let $I$ be the identity matrix that $I(\sigma,\tau)=\mathbbm{1}(\sigma,\tau)$.

\begin{lemma}  \label{lem:Transition-matrix}
  Let $\pi$ be a bipartite distribution, and $P_{RU}$, $P_{AS}$, $T_x$ be defined as above.
  Then we have that
  \begin{enumerate}
    \item $\displaystyle P_{RU}=\frac{I}{2}+\frac{1}{2n}\sum_{x\in V} T_x$;
    \item $\displaystyle P_{AS}=\prod_{i = 1}^{n_1} T_{x_i}\prod_{j=1}^{n_2} T_{y_j}$.
  \end{enumerate}
\end{lemma}
\begin{proof}
  Note that $T_x$ is the transition matrix of resampling $\sigma(x)$.
  For $P_{RU}$, the term $\frac{I}{2}$ comes from the fact that the chain is ``lazy''.
  With the other $1/2$ probability, we resample $\sigma(x)$ for a uniformly chosen $x\in V$.
  This explains the term $\frac{1}{2n}\sum_{x\in V} T_x$.

  For $P_{AS}$, we sequentially resample all variables in $V_1$ and then all variables in $V_2$,
  which yields the expression.
\end{proof}

\begin{lemma}\label{lem:updating-vertex}
  Let $\pi$ be a bipartite distribution and $T_x$ be defined as above.
  Then we have that
  \begin{enumerate}
    \item \label{item:self-adjoint} For any $x\in V$, $T_x$ is a self-adjoint operator and idempotent.
      Namely, $T_x=T_x^*$ and $T_x T_x= T_x$.
    \item \label{item:norm} For any $x\in V$, $\norm{\pi}{T_x}=1$.
    \item \label{item:commute} For any $x,x'\in V_i$ where $i=1$ or $2$, $T_{x}$ and $T_{x'}$ commute.
      In other words $T_{x'} T_x = T_x T_{x'}$ if $x,x' \in V_i$ for $i=1$ or $2$.
  \end{enumerate}
\end{lemma}
\begin{proof}
  For Item \ref{item:self-adjoint},
  the fact that $T_x$ is self-adjoint follows from the detailed balance condition \eqref{eqn:detailed-balance}.
  Idempotence is because updating the same vertex twice is the same as a single update.

  Item \ref{item:norm} follows from Item \ref{item:self-adjoint}.
  This is because
  \begin{align*}
    \norm{\pi}{T_x}=\norm{\pi}{T_xT_x} = \norm{\pi}{T_xT_x^*} = \norm{\pi}{T_x}^2.
  \end{align*}

  For Item \ref{item:commute}, suppose $i=1$.
  Since $\pi$ is bipartite, resampling $x$ or $x'$ only depends on $\sigma_{2}$.
  Therefore the ordering of updating $x$ or $x'$ does not matter as they are in the same partition.
\end{proof}

Define
\begin{align*}
  P_{GS1} := \frac{I}{2}+ \frac{1}{2n_1}\sum_{i=1}^{n_1}T_{x_i}, \text{\quad\quad and \quad\quad} P_{GS2} := \frac{I}{2}+ \frac{1}{2n_2}\sum_{j=1}^{n_2}T_{y_j}.
\end{align*}
Then, since $n_1+n_2=n$,
\begin{align} \label{eqn:GS-GS12}
  P_{RU} = \frac{1}{n}\left( n_1 P_{GS1}+n_2P_{GS2} \right).
\end{align}
Similarly, define
\begin{align*}
  P_{AS1} := \prod_{i=1}^{n_1}T_{x_i}, \text{\quad\quad and \quad\quad} P_{AS2} := \prod_{j=1}^{n_2}T_{y_j}.
\end{align*}
Then 
\begin{align}\label{eqn:SS-12}
  P_{AS} = P_{AS1}P_{AS2}.
\end{align}
With this notation, Lemma \ref{lem:updating-vertex} also implies the following.

\begin{corollary}\label{cor:composition}
  The following holds:
  \begin{enumerate}
    \item \label{item:SS-norm} $\norm{\pi}{P_{AS1}}\le 1$ and $\norm{\pi}{P_{AS2}}\le1$.
    \item \label{item:SS-GS} $P_{AS1}P_{GS1} = P_{AS1}$ and $ P_{GS2}P_{AS2} = P_{AS2}$.
  \end{enumerate}  
\end{corollary}
\begin{proof}
  For Item \ref{item:SS-norm}, by the submultiplicity of operator norms:
  \begin{align*}
    \norm{\pi}{P_{AS1}} & = \norm{\pi}{\prod_{i=1}^{n_1}T_{x_i}} \le \prod_{i=1}^{n_1}\norm{\pi}{T_{x_i}} \\
    & = 1. \tag{By Item \ref{item:norm} of Lemma \ref{lem:updating-vertex}}
  \end{align*}
  The claim $\norm{\pi}{P_{AS2}}\le1$ follows similarly.

  Item \ref{item:SS-GS} follows from Item \ref{item:self-adjoint} and \ref{item:commute} of Lemma \ref{lem:updating-vertex}.
  We verify the first case as follows. 
  \begin{align*}
    P_{AS1}P_{GS1} & = \prod_{i=1}^{n_1}T_{x_i} \left( \frac{I}{2}+ \frac{1}{2n_1}\sum_{j=1}^{n_1}T_{x_j} \right)\\
    & = \frac{1}{2}\cdot \prod_{i=1}^{n_1}T_{x_i} + \frac{1}{2n_1}\cdot \prod_{i=1}^{n_1}T_{x_i}\sum_{j=1}^{n_1}T_{x_j} \\
    & = \frac{1}{2}\cdot \prod_{i=1}^{n_1}T_{x_i} + \frac{1}{2n_1}\cdot \sum_{j=1}^{n_1}T_{x_j}\prod_{i=1}^{n_1}T_{x_i} \\
    & = \frac{1}{2}\cdot \prod_{i=1}^{n_1}T_{x_i} + \frac{1}{2n_1}\cdot \sum_{j=1}^{n_1}T_{x_1}T_{x_2}\cdots T_{x_j}T_{x_j} \cdots T_{x_{n_1}} \tag{By Item \ref{item:commute} of Lemma \ref{lem:updating-vertex}}\\
    & = \frac{1}{2}\cdot \prod_{i=1}^{n_1}T_{x_i} + \frac{1}{2n_1}\cdot \sum_{j=1}^{n_1}\prod_{i=1}^{n_1}T_{x_i} \tag{By Item \ref{item:self-adjoint} of Lemma \ref{lem:updating-vertex}}\\
    & = \frac{1}{2}\cdot \prod_{i=1}^{n_1}T_{x_i} + \frac{1}{2}\cdot \prod_{i=1}^{n_1}T_{x_i}\\
    & = P_{AS1}.
  \end{align*}
  The other case is similar.
\end{proof}

Item \ref{item:SS-GS} of Corollary \ref{cor:composition} captures the following intuition:
if we sequentially update all variables in $V_i$ for $i=1,2$, 
then an extra individual update either before or after does not change the distribution.
Recall Eq.~\eqref{eqn:relaxation-non-reverse}.

\begin{lemma}\label{lem:scan-vs-random}
  Let $\pi$ be a bipartite distribution and $P_{RU}$ and $P_{AS}$ be defined as above.
  Then we have that
  \begin{align*}
    \norm{\pi}{R(P_{AS})-S_{\pi}} \le \norm{\pi}{P_{RU}-S_{\pi}}^2.
  \end{align*}
\end{lemma}
\begin{proof}
  Recall \eqref{eqn:S_pi}, the definition of $S_{\pi}$, using which it is easy to see that 
  \begin{align}\label{eqn:S_pi:commute:SS}
    P_{AS1}S_{\pi} = S_{\pi} P_{AS2} = S_{\pi}S_{\pi}=S_{\pi}.
  \end{align}
  Thus,
  \begin{align}
    P_{AS1}(P_{RU}-S_{\pi})P_{AS2} & = P_{AS1}\left(\frac{n_1}{n}P_{GS1}+\frac{n_2}{n}P_{GS2} -S_{\pi}\right)P_{AS2} \tag{By \eqref{eqn:GS-GS12}}\\
    & = \frac{n_1}{n}P_{AS1}P_{GS1}P_{AS2} + \frac{n_2}{n}P_{AS1}P_{GS2}P_{AS2}-P_{AS1}S_{\pi}P_{AS2} \notag\\
    & = \frac{n_1}{n}P_{AS1}P_{AS2} + \frac{n_2}{n}P_{AS1}P_{AS2}- S_{\pi} \tag{By Item \ref{item:SS-GS} of Cor~\ref{cor:composition}}\\
    & = P_{AS1}P_{AS2} - S_{\pi} \notag \\
    & = P_{AS} - S_{\pi}, \label{eqn:SS-decomp}
  \end{align}
  where in the last step we use \eqref{eqn:SS-12}.
  Moreover, we have that
  \begin{align*}
    P_{AS}^* & = \left( \prod_{i = 1}^{n_1} T_{x_i}\prod_{j=1}^{n_2} T_{y_j} \right)^* \\
    & = \prod_{j=1}^{n_2} T_{y_{n_2+1-j}}^* \prod_{i = 1}^{n_1} T_{x_{n_1+1-i}}^*\\ 
    & = \prod_{j=1}^{n_2} T_{y_{n_2+1-j}} \prod_{i = 1}^{n_1} T_{x_{n_1+1-i}} \tag{By Item \ref{item:self-adjoint} of Lemma \ref{lem:updating-vertex}}\\ 
    & = \prod_{j=1}^{n_2} T_{y_{j}} \prod_{i = 1}^{n_1} T_{x_{i}} \tag{By Item \ref{item:commute} of Lemma \ref{lem:updating-vertex}}\\ 
    & = P_{AS2}P_{AS1}.
  \end{align*}
  Hence, similarly to \eqref{eqn:SS-decomp}, we have that
  \begin{align}
    P_{AS2}(P_{RU}-S_{\pi})P_{AS1} & = P_{AS2}P_{AS1} - S_{\pi} \notag \\
    & = P_{AS}^* - S_{\pi}. \label{eqn:SS-decomp-reverse}
  \end{align}
  Using \eqref{eqn:S_pi:commute:SS}, we further verify that
  \begin{align}
    \left( P_{AS} - S_{\pi} \right)\left( P_{AS}^* - S_{\pi} \right) & = P_{AS}P_{AS}^* - P_{AS}S_{\pi} - S_{\pi}P_{AS}^* +S_{\pi}S_{\pi} \notag\\
    & = P_{AS}P_{AS}^* - S_{\pi} \label{eqn:final-decomp}
  \end{align}
  Combining \eqref{eqn:SS-decomp}, \eqref{eqn:SS-decomp-reverse}, and \eqref{eqn:final-decomp}, we see that
  \begin{align*}
    \norm{\pi}{R(P_{AS})-S_{\pi}} & = \norm{\pi}{ P_{AS}P_{AS}^* - S_{\pi} } \\
    & = \norm{\pi}{\left( P_{AS} - S_{\pi} \right)\left( P_{AS}^* - S_{\pi} \right)}\\
    & = \norm{\pi}{P_{AS1}\left( P_{RU} - S_{\pi} \right)P_{AS2}P_{AS2}\left( P_{RU} - S_{\pi} \right)P_{AS1}}\\
    & \le \norm{\pi}{P_{AS1}}\norm{\pi}{P_{RU} - S_{\pi}}\norm{\pi}{P_{AS2}}\norm{\pi}{P_{AS2}}\norm{\pi}{P_{RU} - S_{\pi}}\norm{\pi}{P_{AS1}}\\
    & \le \norm{\pi}{P_{RU}-S_{\pi}}^2,
  \end{align*}
  where the first inequality is due to the submultiplicity of operator norms,
  and we use Item \ref{item:SS-norm} of Corollary~\ref{cor:composition} in the last line.
\end{proof}
\begin{remark}
  The last inequality in the proof of Lemma~\ref{lem:scan-vs-random}
  crucially uses the fact that the distribution is bipartite.
  If there are, say, $k$ partitions, then the corresponding operators $P_{AS1},\dots,P_{ASk}$ do not commute and the proof does not generalize.  
\end{remark}

\begin{proof}[Proof of Theorem \ref{thm:AS}]
  For the first part, notice that the alternating-scan sampler is aperiodic.
  Any possible state $\sigma$ of the chain must be in the state space $\Omega$.
  Therefore $\pi(\sigma)>0$ and the probability of staying at $\sigma$ is strictly positive.
  Moreover, any single variable update can be simulated in the scan sampler, with small but strictly positive probability.
  Hence if the random-update sampler is irreducible, then so is the scan sampler.

  To show that $T_{rel}(P_{AS})\le T_{rel}(P_{RU})$,
  we have the following
  \begin{align*}
    T_{rel}(P_{AS}) & = \frac{1}{1-\sqrt{1-\lambda(R(P_{AS}))}} \tag{By \eqref{eqn:relaxation-non-reverse}}\\
    & = \frac{1}{1-\sqrt{\norm{\pi}{R(P_{AS})-S_{\pi}}}} \tag{By \eqref{eqn:spectral-operator}} \\
    & \le \frac{1}{1-\norm{\pi}{P_{RU}-S_{\pi}}} \tag{By Lemma \ref{lem:scan-vs-random}}\\
    & = \frac{1}{\lambda(P_{RU})} \tag{By \eqref{eqn:spectral-operator}}\\
    & = T_{rel}(P_{RU}). \tag{By \eqref{eqn:relaxation-time}}
  \end{align*}
  This completes the proof.
\end{proof}

\begin{remark}
  It is easy to check that the proof also works if we consider the non-lazy version of $P_{RU}$.
  To do so, we just replace $\frac{I}{2}+\frac{1}{2n}\sum_{x\in V} T_x$ with $\frac{1}{n}\sum_{x\in V} T_x$ and the rest of the proof goes through without changes.
\end{remark}

\begin{remark}
  Our argument can also handle the case of general state spaces, such as Gaussian variables,
  since the essential property we use is the commutativity of updating variables from the same partition.
  For general state spaces, in order to apply Theorem \ref{thm:AS} on mixing times,
  we need to replace Theorem \ref{thm:relaxation} and Theorem \ref{thm:relaxation-reverse} with their continuous counterparts.
  See for example~\citep{LS88}.
\end{remark}

Using Theorem \ref{thm:relaxation} and Theorem \ref{thm:relaxation-reverse}, we translate Theorem \ref{thm:AS} in terms of the mixing time.

\begin{corollary}\label{cor:AS}
  For a Markov random field defined on a bipartite graph,
  let $P_{RU}$ and $P_{AS}$ be the transition kernels of the random-update Gibbs sampler and the alternating-scan sampler, respectively.
  Then,
  \begin{align*}
    T_{mix}(P_{AS})\le \log \left(  \frac{4e^2}{\pi_{min}}\right)\left( T_{mix}(P_{RU})+1 \right),
  \end{align*}
  where $\pi_{min}=\min_{\sigma\in \Omega}\pi(\sigma)$.
  \label{corollary:mixing}
\end{corollary}

Since $n$ variables are updated in each epoch of $P_{AS}$,
one might hope to strengthen Theorem \ref{thm:AS} so that $n T_{rel}(P_{AS})$ is also no larger than $T_{rel}(P_{RU})$.
Unfortunately, this is not the case and we give an example (similar to the ``two islands'' example due to \cite{HDMR16}) 
where $T_{mix}(P_{AS})\asymp T_{mix}(P_{RU})$ and $T_{rel}(P_{AS})\asymp T_{rel}(P_{RU})$.
This example implies that Theorem \ref{thm:AS} is asymptotically tight.
However, it is still possible that Corollary \ref{cor:AS} is loose by a factor of $\log\pi_{min}^{-1}$.
This potential looseness is difficult to circumvent due to the spectral approach we took.

\begin{example}\label{exm:bipartite-complete}
  Let $G=(L\cup R,E)$ be a complete bipartite graph $K_{n,n}$ and we want to sample an uniform independent set in $G$.
  In other words, each vertex is a Boolean variable and a valid configuration is an independent set $I\subseteq L\cup R$.
  To be an independent set in $K_{n,n}$, $I$ cannot intersect both $L$ and $R$.
  Hence the state space is $\Omega=\{I\mid I \subseteq L\text{ or }I\subseteq R\}$ and the measure $\pi$ is uniform on $\Omega$.
  Under single-site updates, $\Omega$ is composed of two independent copies of the Boolean hypercube $\{0,1\}^n$ with the two origins identified.
  The random-update Gibbs sampler has mixing time $O(2^n)$ 
  because the (maximum) hitting time of the Boolean hypercube is $O(2^n)$ and the mixing time is upper bounded by the hitting time multiplied by a constant \cite[Eq.\ (10.24)]{LPW06}.
  The relaxation time is also $O(2^n)$ by Theorem \ref{thm:relaxation}.
  In fact, it is not hard to see that both quantities are $\Theta(2^n)$.

  On the other hand, the alternating-scan sampler has mixing time $\Omega(2^n)$ and relaxation time $\Omega(2^n)$.
  For the mixing time, we partition the state space $\Omega$ into $\Omega_L=\{I\mid I\subset L\}$ and $\Omega_R=\{I\mid I\subset R \text{ and } I\neq\emptyset\}$.
  Consider the alternating scan projected down to $\Omega_L$ and $\Omega_R$.
  If the current state is in $\Omega_L$, then there is $2^{-n}$ probability to go to $\emptyset$ after updating all vertices in $L$,
  and then with probability $1-2^{-n}$ the state goes to $\Omega_R$ after updating all vertices in $R$.
  Similarly, going from $\Omega_R$ to $\Omega_L$ has also probability $O(2^{-n})$.
  Thus in each epoch of the alternating scan, the probability to go between $\Omega_L$ and $\Omega_R$ is $\Theta(2^{n})$ and the mixing time is thus $\Theta(2^{-n})$.
  The relaxation time can be similarly bounded using a standard conductance argument \citep{Sin92}.

  In summary, for this bipartite distribution $\pi$, we have that $T_{rel}(P_{AS})\asymp T_{rel}(P_{RU})$ and $T_{mix}(P_{AS})\asymp T_{mix}(P_{RU})$.
  Therefore, Theorem \ref{thm:AS} is asymptotically tight and Corollary \ref{cor:AS} is tight up to the factor $\log\pi_{\min}^{-1}$.
\end{example}

We conjecture that the factor $\log\pi^{-1}_{min}$ should not be in Corollary \ref{cor:AS}.
However, this factor is inherently there with the spectral approach.
To get rid of it a new approach is required.

We note that in Example~\ref{exm:bipartite-complete},
alternating scan is not necessarily the best scan order.
Indeed, as shown by \cite{HDMR16}, if we scan vertices alternatingly from the left and right, 
rather than scanning variables layerwise,
the mixing time is smaller by a factor of $n$.
Thus, although Theorem \ref{thm:AS} and Corollary \ref{cor:AS} provide certain guarantees of the alternating-scan sampler,
the layerwise alternating order is not necessarily the best one.

\section{Bipartite Distributions in Machine Learning}\label{sec:implication}

The results we developed so far can be
applied to probabilistic graphic
models with bipartite structures, most
notably Restricted Boltzmann Machines (RBM)
and Deep Boltzmann Machines (DBM).
Although real-world systems for
RBM and DBM inference rely on layerwise
systematic scans, we are the first
to provide a theoretical justification of 
such implementations.

\subsection{Markov Random Fields}

A Markov random field (MRF) with binary factors
$\inner{}{G,S,\pi}$ is defined on a graph $G=(V,E)$, 
where each edge describes a ``factor'' $f_e$ and each vertex is a variable drawing from $S$, a set of possible values.
Each factor is a function $S^2\to \mathbb{R}$.
A configuration $\sigma\in S^V$ is a mapping from $V$ to $S$.
In addition, each vertex is equipped with a factor $g_v:S\to \mathbb{R}$.
Let $\Omega\subseteq S^V$ be the state space, which is usually defined by a set of hard constraints.
When there is no hard constraint, the state space $\Omega$ is simply $S^V$.
The Hamiltonian of $\sigma\in\Omega$ is defined as
\begin{align*}
  H(\sigma) = \sum_{e=(u,v)\in E}f_e(\sigma(u),\sigma(v))+\sum_{v\in V}g_v(\sigma(v)).
\end{align*}
The Gibbs distribution $\pi(\cdot)$ is defined as
  $\pi(\sigma) \propto \mathbbm{1}(\sigma\in \Omega)\exp(H(\sigma))$.
These models are popularly used in applications such as image processing~\citep{Li:2009:Book}
and natural language processing~\citep{Lafferty:2001:ICML}.

It is easy to check that, when the underlying graph $G$ is bipartite, the Gibbs distribution is bipartite in the sense of Definition \ref{def:bipartite}.
Thus Theorem \ref{thm:AS} and Corollary \ref{cor:AS} apply to this setting.

\subsection{Restricted/Deep Boltzmann Machines}

Restricted Boltzmann Machines (RBM) was introduced by \cite{Smolensky:1986:IPD}.
It is a special case of the general MRF 
in which all variables are Boolean (i.e., $S = \{0, 1\}$) and are partitioned into two disjoint sets, $V_1$ and $V_2$. 
There is a factor between each variable in $V_1$ and $V_2$, and the Hamiltonian is 
\begin{align*}
  H(\sigma) = \sum_{u \in V_1, v \in V_2} 
  W_{uv} \sigma(u)\sigma(v)
  + \sum_{v\in V} W_v \sigma(v).
\end{align*}
where $W_{uv}$ and $W_v$ are real-valued weights. 
Figure~\ref{fig:rbmdbm}(a) illustrates the structure of RBMs.
We use $[f_{00},f_{01},f_{10},f_{11}]$ to describe a general binary factor defined on Boolean variables.
Thus, $[0,0,0,W]$ denotes a standard RBM factor with weight $W$, and $[W,0,0,W]$ denotes an Ising model with weight $W$ (after some renormalization).

\begin{figure*}[t]
\centering
\includegraphics[width=0.8\textwidth]{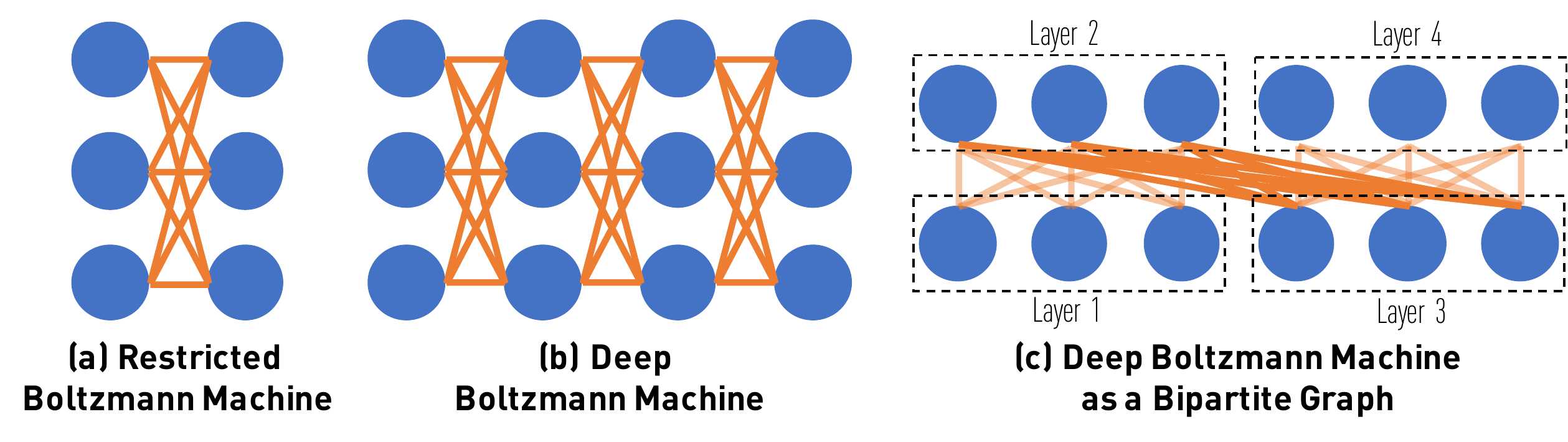}
\caption{Restricted Boltzmann machines and deep Boltzmann machines as bipartite systems.}
\label{fig:rbmdbm}
\end{figure*}

Markov chain Monte Carlo is a common approach to perform inference for RBMs,
which involves sampling a configuration from the Gibbs distribution $\pi$. 
The de facto algorithm for this task is Gibbs sampling,
in which the conditional probability of each step can be calculated from only the Hamiltonian. 
In this context, the alternating-scan algorithm we study corresponds to a {\em layerwise scan} --- 
first update all variables in $V_1$ and then all variables in $V_2$. 
This scan order allows one to use efficient linear algebra primitives such as dense matrix multiplication implemented with GPUs or SIMD instructions on modern CPUs.

Deep Boltzmann Machines (DBM), introduced by \cite{Salakhutdinov:2009:AISTATS}, is a Deep Learning model that extends RBM to multiple layers as illustrated in Figure~\ref{fig:rbmdbm}(b). 
This layer structure is indeed bipartite, shown in Figure~\ref{fig:rbmdbm}(c). 
The scan order induced is thus to update odd layers first and even ones after.
Like most deep learning models, the scan (evaluation) order of variables has significant impact on the speed and performance of the system. 
The layerwise implementation is particularly advantageous thanks to dense linear algebra primitives.

Given an RBM or DBM with $n$ variables, it is easy to see that $\log \pi_{\min}^{-1}$ is $O(n)$.
Thus, Corollary~\ref{corollary:mixing} implies that, comparing to the random-update algorithm,
the layerwise systematic scan algorithm incurs at most a $O(n^2)$ slowdown in the convergence rate.
This comparison result improves exponentially (in the worst case) upon previous result \citep{HDMR16}.

\section{Experiments}\label{sec:experiments}

Empirically evaluating the mixing time of Markov chains is notoriously difficult.
In general, it is hard under certain complexity assumptions \citep{BBM11} and lower bounds have been established for more concrete settings by \cite{HKS15} 
(see also \citep{HKS15} for a comprehensive survey on this topic).
We evaluate the mixing time in either exact and straightforward or approximate but tractable ways, including 
(1) calculating directly using the transition matrix for small graphs, 
(2) taking advantage of symmetries in the state space for medium-sized graphs, 
and (3) using the coupling time (defined later) as a proxy of the mixing time for large graphs.

\para{Mixing Time on Small Graphs}
We evaluate the mixing time in a brutal force way, 
namely, we multiply the transition matrix until the total variation distance to the stationary distribution is below the threshold.
Since the state space is exponentially large, such a method is only feasible in small graphs.

Figure~\ref{fig:mixing_rbm} and 
Figure~\ref{fig:mixing_dbm} contains 
the comparison of the mixing time
for small graphs (RBMs of up to 12
variables and DBMs with 4 layers and
3 variable per layer). 
We vary
(1) number of variables, (2)
factor functions (shown as
the entries of truth table 
in the caption), or (3) 
the weight of factors, in different
figures and report the mixing times
of random updates and layerwise scan.
All solid lines count 
mixing time in \# variable updates
and the dotted line in \# epochs. 

\begin{figure*}[t]
\centering
\includegraphics[width=0.9\textwidth]{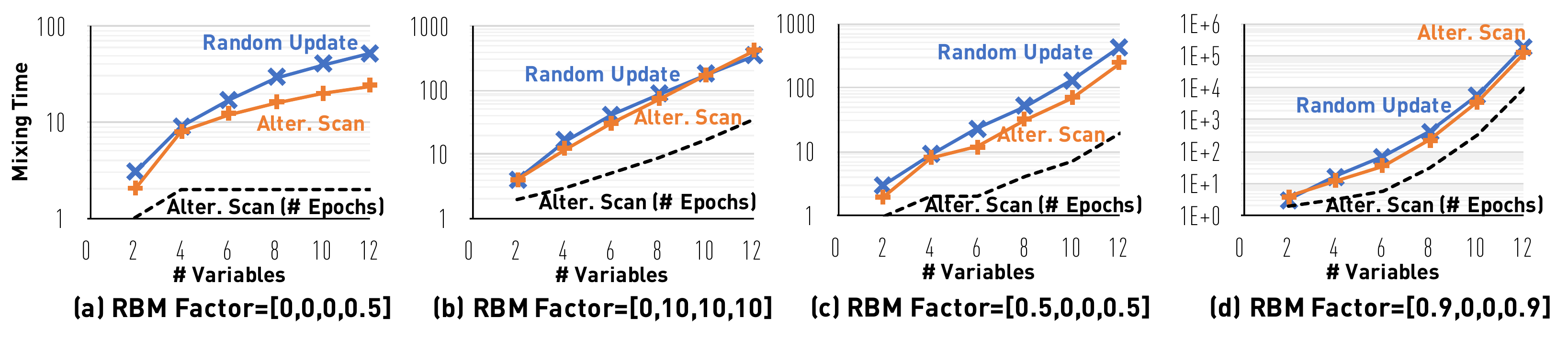}
\caption{Mixing time of Gibbs samplers on 
Restricted Boltzmann Machines. See Section~\ref{sec:experiments}.}
\label{fig:mixing_rbm}
\end{figure*}

\begin{figure*}[t]
\centering
\includegraphics[width=0.7\textwidth]{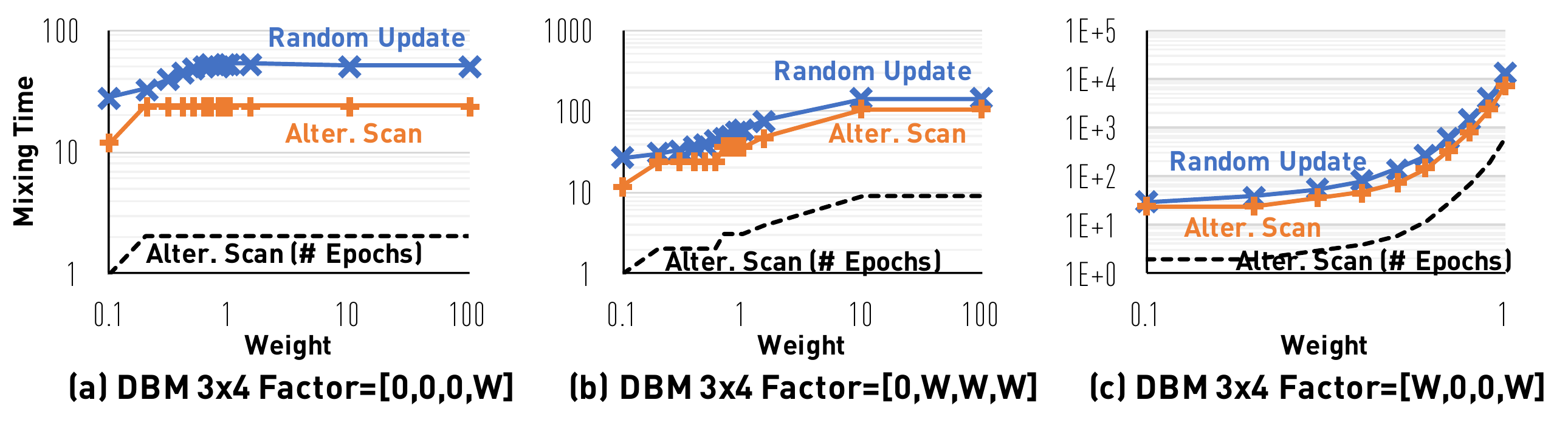}
\caption{Mixing time of Gibbs samplers on 
Deep Boltzmann Machines. See Section~\ref{sec:experiments}.}
\label{fig:mixing_dbm}
\end{figure*}

We see that, empirically, alternating scan has comparable, sometimes better, mixing time than random updates,
even when counting in the number of variable updates. 
On one hand, it confirms our result that the mixing time of alternating scan and random updates are similar. 
On the other, it shows that our result, although asymptotically tight for the worst case, is not ``instance optimal''. 
This observation indicates promising future direction for beyond-worst case analysis.


\para{Medium-sized Graphs}
We now turn to Example \ref{exm:bipartite-complete}, which has also been studied by \cite{HDMR16} and is asymptotically the worst case of Theorem \ref{thm:AS}. 
Due to certain symmetries, we have a much more succinct representation of the state space, 
and manage to calculate the mixing and relaxation times for mildly larger graphs (up to 50 variables). 
As illustrated in Figure \ref{fig:mixing_indset}, the alternating-scan sampler is slower than, but still comparable to the random-update sampler.
This is consistent with the discussion in Example \ref{exm:bipartite-complete}.

\begin{figure}[h]
\centering
\includegraphics[width=0.5\textwidth]{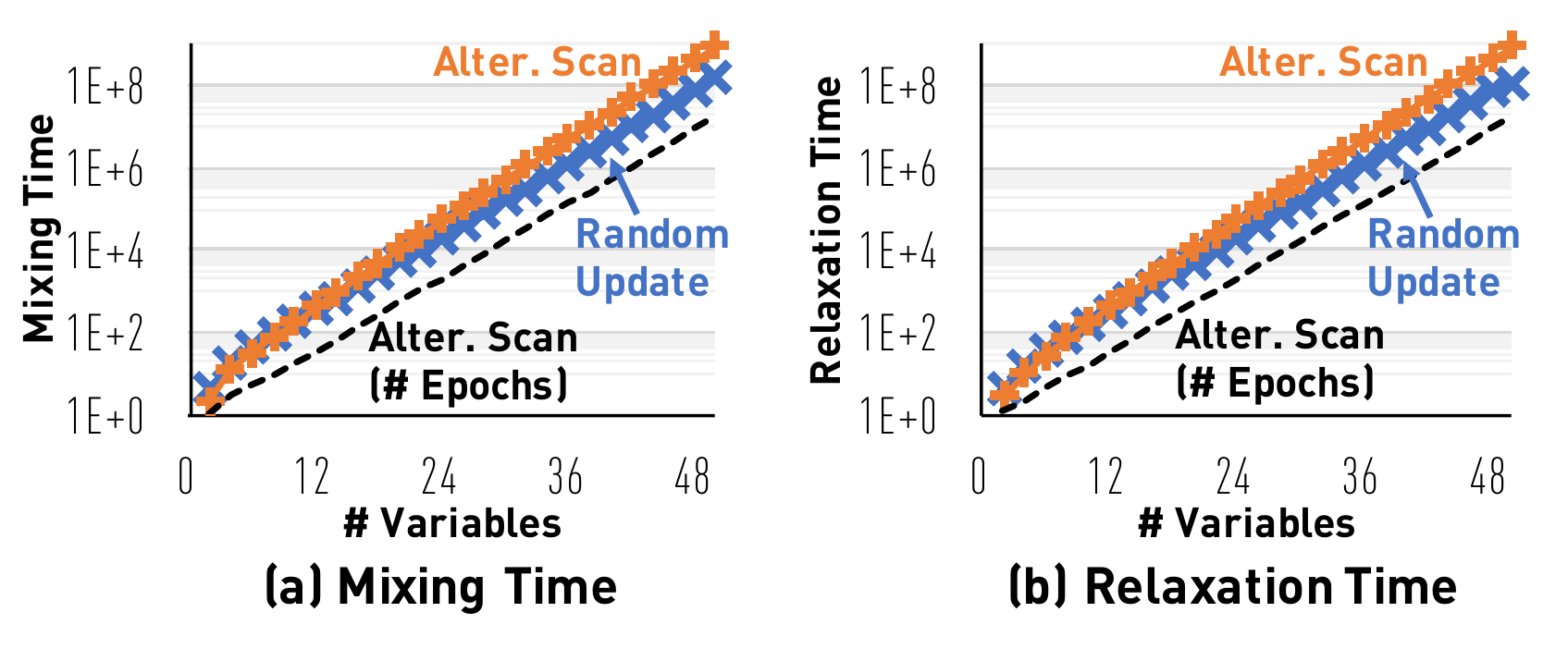}
\caption{Mixing time comparison on medium-sized graphs.}
\label{fig:mixing_indset}
\end{figure}


\para{Coupling Time on Large Graphs} 
Lastly, we use the coupling time as a proxy of the mixing time and
estimate it on large graphs with $10^4$ variables and $5\times 10^4$ randomly chosen factors. 

We use the grand coupling \citep[Chapter~5]{LPW06}.
Let $T_{\sigma,\tau}$ be the first time two copies of the same Markov chain meet, 
with initial states $\sigma$ and $\tau$, under certain coupling. 
Then the coupling time is $\max_{(\sigma,\tau)\in\Omega^2} T_{\sigma,\tau}$. 
All of the models we tested are monotone~\citep{PW13}, 
in which the coupling time under the grand coupling can be easily evaluated by simulating from the top and bottom states.
The coupling time is closely related to the mixing time~\citep[Chapter~5]{LPW06}.
In particular, it is an upper bound of the mixing time regardless of the coupling,
and designing a good coupling is an important technique to prove rapid mixing \citep{BD97}.
Our experimental findings are summarized in Figure \ref{fig:coupling}.

\begin{figure}[h]
\centering
\includegraphics[width=0.5\textwidth]{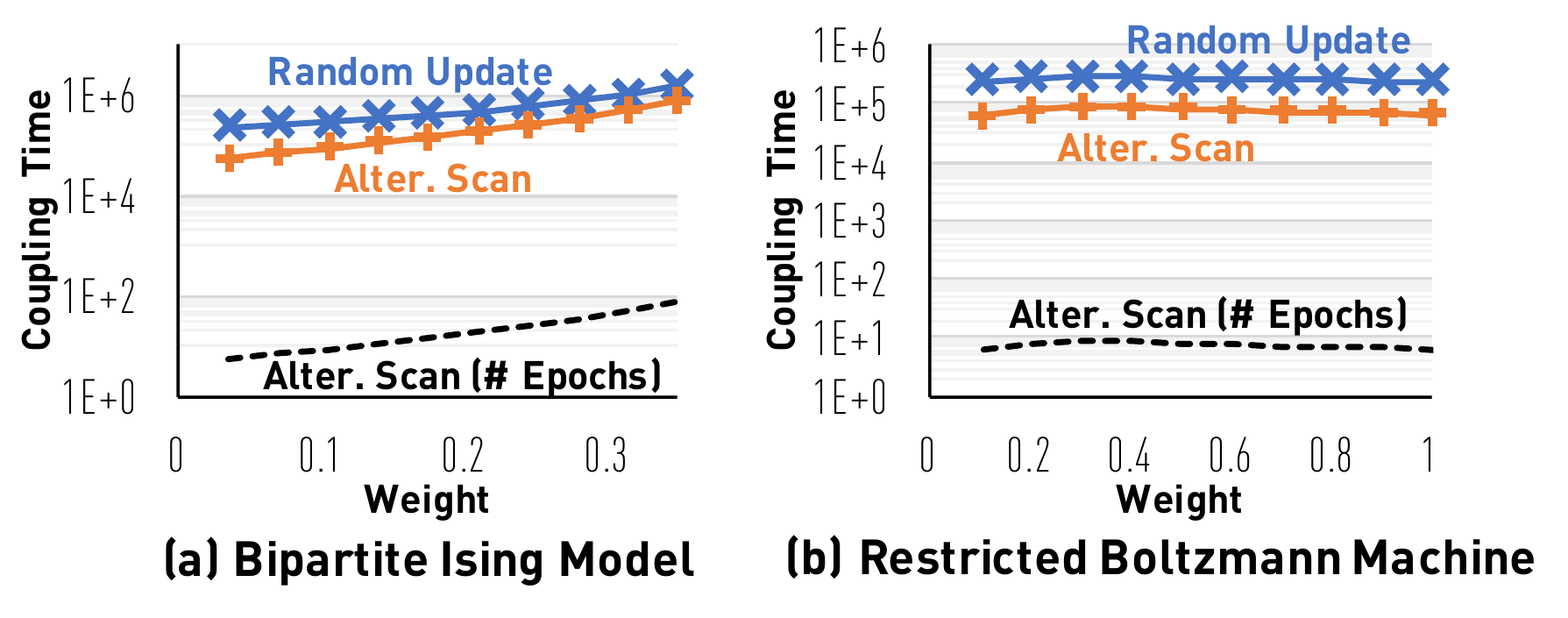}
\caption{Coupling time comparison on large and random graphs.}
\label{fig:coupling}
\end{figure}

In these experiments, we choose our parameters to stay within the rapidly mixing regime \citep{MS13} and avoid exponential mixing times.
As we can see in Figure \ref{fig:coupling}, alternating scan is faster than random updates (in terms of variable updates).
Indeed, numerical evidence suggests that the speedup factor is close to $2$.

\section{Concluding Remarks}

In summary, we have shown that for a bipartite distribution, 
the relaxation time of the alternating-scan sampler (in terms of epochs) is no larger than that of the random-update sampler. 
This is asymptotically tight and implies a (weaker) comparison result on the mixing time.
Future directions include more fine-grained comparison results, 
and going beyond bipartite distributions.

\bibliographystyle{plainnat}
\bibliography{Systematic-Scan}

\end{document}